\documentclass[11pt]{article}
\usepackage[margin=1in]{geometry}
\usepackage{amsmath,amssymb,amsthm}
\usepackage{graphicx}
\usepackage{booktabs}
\usepackage{hyperref}
\usepackage{algorithm}
\usepackage{algpseudocode}
\usepackage{xcolor}
\usepackage{subcaption}
\usepackage{multirow}

\newtheorem{proposition}{Proposition}

\title{\textbf{ANML: Attribution-Native Machine Learning\\with Guaranteed Robustness}}

\author{
Oliver Zahn$^1$, Matt Beton$^2$, Simran Chana$^{2}$\\
\\
$^1$Independent Researcher\\
$^2$University of Cambridge
}

\date{February 2026}

\begin{document}
\maketitle

\begin{abstract}

Frontier AI systems increasingly train on specialized expert data, from clinical records to proprietary research to curated datasets, yet current training pipelines treat all samples identically. A Nobel laureate's contribution receives the same weight as an unverified submission. We introduce ANML (Attribution-Native Machine Learning, pronounced ``animal''), a framework that weights training samples by four quality factors: gradient-based consistency ($q$), verification status ($v$), contributor reputation ($r$), and temporal relevance ($T$). By combining what the model observes (gradient signals) with what the system knows about data provenance (external signals), ANML produces per-contributor quality weights that simultaneously improve model performance and enable downstream attribution.

Across 5 datasets (178--32,561 samples), ANML achieves 33--72\% error reduction over gradient-only baselines. Quality-weighted training is data-efficient: 20\% high-quality data outperforms 100\% uniformly weighted data by 47\%. A Two-Stage Adaptive gating mechanism guarantees that ANML never underperforms the best available baseline, including under strategic joint attacks combining credential faking with gradient alignment. When per-sample detection fails against subtle corruption, contributor-level attribution provides 1.3--5.3$\times$ greater improvement than sample-level methods, with the advantage growing as corruption becomes harder to detect.

\end{abstract}

\section{Introduction}

Training data for frontier AI increasingly comes from specialized expert sources: research labs, clinical datasets, domain-specific corpora, and curated institutional collections. These sources vary widely in reliability. A dataset validated through peer review and independent replication carries different epistemic weight than unverified data from an anonymous contributor. Yet current ML systems treat all training samples identically, discarding quality signals that are available whenever data comes from identifiable sources.

This gap matters beyond model performance. Scientific knowledge flows through chains of reasoning that span domains and decades---a theorem in combinatorial chemistry may enable a drug target twenty years later, which in turn trains a predictive model. The contributors along that chain are invisible to downstream systems \cite{fortunato2018science}. Without a mechanism to weight contributions by source quality and track how knowledge propagates, we cannot build training pipelines that reward the expertise they consume.

\subsection{The Quality Measurement Opportunity}

When training data come from identifiable contributors, we gain access to rich contextual signals: Has this contributor been verified? What is their track record? Were these data recently collected or potentially stale? Such signals can correlate with data quality and could inform how samples are weighted during training.

Current approaches ignore this opportunity. Byzantine-robust methods like Krum \cite{blanchard2017machine}, Trimmed Mean \cite{yin2018byzantine}, and Bulyan \cite{mhamdi2018hidden} analyze gradient geometry to identify outliers. However, a gradient from a Nobel laureate's carefully validated dataset looks identical to one from an anonymous, unverified source. They cannot distinguish quality among legitimate contributors.

\subsection{The Attribution Imperative}

When models generate value from contributed knowledge, contributors deserve compensation proportional to their contribution quality. This requires tracking how each sample influences the final model. Data Shapley \cite{ghorbani2019data} provides a principled approach using cooperative game theory, computing each sample's marginal contribution - though exact computation requires $O(2^n)$ model evaluations, making it computationally prohibitive at scale. Influence Functions \cite{koh2017understanding} offer an alternative by approximating leave-one-out effects through Hessian-vector products, enabling efficient per-sample attribution.

However, these valuation methods weight all samples equally regardless of provenance. A poisoned sample that happens to reduce training loss would receive positive attribution - clearly undesirable for economic settlement. High-quality expert contributions do not receive a premium over unverified data, missing the economic signal that would incentivize quality.

\subsection{The Gap: Quality Signals Disconnected from Training}

Existing approaches create a fragmented landscape. Gradient-based methods detect anomalies, but cannot assess contributor quality. Valuation methods compute attribution, but ignore quality differentiation. Trust systems exist in distributed computing but are rarely integrated with ML training. To our knowledge, no existing framework unifies these three concerns.

\subsection{Our Contribution: ANML}

We introduce \textbf{Attribution-Native Machine Learning (ANML, pronounced ``animal'')}, a framework that integrates quality signals directly into model training through multi-factor sample weighting. ANML defines four signal families. The first, $q_i$, captures gradient-based quality that measures statistical consistency with other samples; Krum \cite{blanchard2017machine} is a special case where $v = r = T = 1$. The second, $v_i$, represents verification status of the data itself - whether it has been peer-reviewed, reproduced, or flagged. The third, $r_i$, encodes the reputation of the contributor based on their historical track record. The fourth, $T_i$, applies temporal decay reflecting data freshness.

A key design question is \emph{how} to combine these signals. Multiplicative combination ($w = q \times v \times r \times T$) is intuitive but amplifies noise, performing up to 85\% \emph{worse} than Krum when signals are unreliable. We analyze two robust alternatives. \emph{Two-Stage Adaptive Gating} uses a homogeneity check (skip selection when all signals are high) followed by correlation-based fallback to pure $q$ when signals disagree, guaranteeing optimal baseline performance across scenarios. \emph{Softmax Blend} computes $w = \alpha \cdot \text{softmax}(q) + (1-\alpha) \cdot \text{softmax}(v \cdot r)$ where $\alpha \in [0,1]$ controls the balance between gradient-based and external signals (we use $\alpha = 0.5$).

Both methods achieve similar overall performance (+25-45\% vs uniform weighting at realistic quality signal levels). The choice depends on priorities: Adaptive for guaranteed baseline performance, Softmax for smooth interpolation.

Our experiments across 5 UCI datasets (178 to 32,561 samples) demonstrate that ANML achieves substantial quality leverage: +72\% improvement on image data and +33--75\% on tabular data by weighting samples according to contributor quality. ANML also enables superior data efficiency, with 20\% high-quality data beating uniform weighting of 100\% by 47\%. Quality differentiation among legitimate contributors of varying expertise yields +44\% improvement. As a secondary benefit, ANML provides inherent robustness to data poisoning, including rescue when gradient-only methods fail entirely. Federated experiments validate that contributor-level attribution provides 1.3--5.3$\times$ greater benefit than sample-level when per-sample detection is unreliable.

\paragraph{Contributions.} First, we formalize multi-factor sample weighting that integrates gradient-based quality with external contributor signals (\S\ref{sec:framework}). Second, we analyze three signal combination methods and identify conditions under which each is preferable (\S\ref{sec:combination}). Third, we empirically validate ANML across five datasets, characterizing the signal quality requirements, data efficiency gains, and robustness properties (\S\ref{sec:experiments}). Fourth, we show that contributor-level attribution provides 1.3--5.3$\times$ greater benefit than sample-level when per-sample detection is unreliable (\S\ref{sec:contributor}).

\paragraph{Roadmap.} Section~\ref{sec:related} reviews related work on Byzantine-robust aggregation, data valuation, and quality-aware training. Section~\ref{sec:framework} presents the ANML framework. Section~\ref{sec:combination} analyzes signal combination methods. Section~\ref{sec:experiments} validates our claims empirically. Section~\ref{sec:discussion} discusses implications and limitations.

\section{Related Work}
\label{sec:related}

\subsection{Byzantine-Robust Aggregation}

The seminal work of Blanchard et al. \cite{blanchard2017machine} introduced Krum, which selects the gradient vector closest to its $n-f-2$ nearest neighbors, where $n$ is total workers and $f$ is the maximum number of Byzantine workers. Multi-Krum extends this by selecting multiple gradients per round.

\textbf{Coordinate-wise methods:} Yin et al. \cite{yin2018byzantine} proposed coordinate-wise median and trimmed mean, which aggregate each coordinate independently. These achieve optimal statistical rates under certain conditions but can be vulnerable to coordinated attacks.

\textbf{Meta-aggregation:} Bulyan \cite{mhamdi2018hidden} combines Krum selection with coordinate-wise trimming, providing stronger guarantees against sophisticated attacks.

\textbf{Trust bootstrapping:} FLTrust \cite{cao2022fltrust} uses a small trusted root dataset to bootstrap trust scores via cosine similarity between client and server updates. RECESS \cite{yan2023recess} extends this with multi-iteration trust accumulation using proactive ``test gradients.'' Both methods derive trust from server-side computation; ANML instead uses external contributor metadata (verification status, reputation) that exists independently of the training process.

ANML uses Krum as the $q$ factor, making Byzantine methods a \emph{component} rather than a competitor. When $v = r = T = 1$, ANML reduces exactly to Krum. Unlike FLTrust and RECESS, ANML uses external signals rather than server-derived trust, provides correlation-based adaptive fallback guaranteeing baseline performance, and explicitly characterizes when trust signals help (requiring $>$50\% correlation).

\subsection{Data Valuation}

Data Shapley \cite{ghorbani2019data} computes each sample's marginal contribution using Shapley values from cooperative game theory. The approach is principled but computationally demanding: exact computation requires evaluating $O(2^n)$ coalitions. For $n = 20$ samples, this is roughly $10^6$ evaluations; for $n = 50$, it reaches $10^{15}$---requiring petascale resources for modest dataset sizes \cite{castro2009polynomial, bachrach2010approximating}. Approximation methods \cite{jia2019efficient} improve tractability but remain expensive for large-scale training. Beta Shapley \cite{kwon2022beta} generalizes this with improved noise reduction, while Data Banzhaf \cite{wang2023data} offers greater robustness to stochasticity in ML utility functions.

Beyond computational cost, Shapley-based approaches face structural limitations in our setting. They assume a fixed utility function (e.g., model accuracy), whereas scientific contributions have heterogeneous, domain-dependent value. They lack temporal dynamics---a 50-year-old foundational result and yesterday's incremental experiment receive equal treatment if their marginal contributions are identical. And they provide no mechanism for incorporating source quality: a poisoned sample that happens to reduce training loss receives positive valuation.

Influence Functions \cite{koh2017understanding} approximate leave-one-out effects using Hessian-vector products. Recent work like TRAK \cite{park2023trak} scales attribution to ImageNet and BERT-scale models, achieving 100$\times$ speedup over comparable methods.

For federated settings, GTG-Shapley \cite{liu2022gtg} enables efficient client contribution evaluation without repeated model training.

Recent work on data markets \cite{agarwal2019marketplace} explores pricing mechanisms, but these typically assume honest participants.

The multi-factor weight $w_i$ in ANML serves dual purposes: (a) training emphasis and (b) a first-pass attribution signal for economic settlement. Unlike pure valuation methods, ANML weights incorporate quality signals, preventing low-quality samples from receiving positive attribution. We emphasize that $w_i$ is \emph{not} a marginal contribution in the Shapley sense---it does not measure ``how much would performance change if this sample were removed.'' Rather, it is a provenance-and-quality-aware weighting that can serve as a prior or multiplier when combined with influence-based or Shapley-style valuation methods. This distinction matters: ANML identifies \emph{who contributed quality data}, while Shapley identifies \emph{whose data mattered for the final model}. In practice, both signals are useful for fair settlement.

\subsection{Trust and Reputation Systems}

Distributed systems have long used reputation mechanisms \cite{kamvar2003eigentrust}. Recent work applies these to federated learning \cite{kang2019incentive}, but typically as a separate layer rather than integrated with training. ANML formalizes how reputation ($r$) and verification ($v$) signals should be integrated with gradient-based quality ($q$), with theoretical guarantees via Adaptive ANML.

\subsection{Quality-Aware Training}

Meta-learning approaches like Learning to Reweight \cite{ren2018learning} learn sample weights from a clean validation set, while DivideMix \cite{li2020dividemix} uses per-sample loss to dynamically separate clean from noisy samples. GLISTER \cite{killamsetty2021glister} selects quality-aware subsets via bi-level optimization, achieving 3-6$\times$ training speedups.

These methods learn quality signals from data alone. ANML complements them by incorporating \emph{external} quality signals - verification and reputation - that capture information unavailable from gradient analysis, such as contributor credentials and institutional verification.

\section{ANML Framework}
\label{sec:framework}

\paragraph{Notation.} Table~\ref{tab:notation} summarizes key symbols used throughout.

\begin{table}[h]
\centering
\caption{Notation Summary}
\label{tab:notation}
\begin{tabular}{cl}
\toprule
Symbol & Meaning \\
\midrule
$q$ & Gradient-based quality score (Krum-style) \\
$v$ & Verification factor (contribution validated?) \\
$r$ & Reputation factor (contributor track record) \\
$T$ & Temporal decay factor \\
$s$ & External signal: $s = v \odot r$ \\
$\rho_c$ & Correlation between $q$ and $s$ (for gating) \\
$\rho_d$ & Detectability: correlation of loss with corruption \\
$\tau_c$ & Correlation threshold for Adaptive gating \\
$\tau_s$ & Softmax temperature \\
$\alpha$ & Blend weight between gradient and external signals \\
$k$ & Number of neighbors in Krum scoring \\
\bottomrule
\end{tabular}
\end{table}

\subsection{Problem Setting}

We consider a training dataset $\mathcal{D} = \{(x_i, y_i)\}_{i=1}^n$ where each sample is contributed by a potentially different party $c_i$. Contributors vary in expertise, verification status, and track record - signals that correlate with data quality. Additionally, a fraction $\beta$ of samples may be malicious (we use label-flipping attacks as a stress test).

Our goal is to learn model parameters $\theta^*$ that maximize model performance by appropriately weighting samples according to quality, while simultaneously producing attribution weights $w_i$ that reflect each contributor's data quality.

\subsection{Multi-Factor Weighting Formula}

ANML computes sample weights as the product of four factors:
\begin{equation}
w_i = q_i \times v_i \times r_i \times T_i
\label{eq:anml-main}
\end{equation}
Each factor is normalized to $[0, 1]$ and captures a distinct signal about sample quality, as described below.

\subsubsection{Quality Factor ($q$): Gradient-Based Detection}

We use Krum-style scoring as the data-driven quality signal.\footnote{In our experiments, we treat each training sample as a separate ``worker'' and apply Krum's distance-based scoring to per-sample gradients. This is equivalent to federated learning with $n$ single-sample clients, making our results directly applicable to both centralized and federated settings.} For each sample $i$, we compute gradients $g_i = \nabla_\theta \mathcal{L}(x_i, y_i; \theta)$ and score based on distance to neighbors:
\begin{equation}
s_i = \sum_{j \in N_k(i)} \|g_i - g_j\|^2
\end{equation}
where $N_k(i)$ denotes the $k$ nearest neighbors of sample $i$ in gradient space. Lower scores indicate gradients consistent with the majority - likely clean samples.

We normalize to obtain quality scores:
\begin{equation}
q_i = 1 - \frac{s_i - s_{\min}}{s_{\max} - s_{\min}}
\end{equation}

When $v = r = T = 1$ for all samples, ANML with this $q$ reduces exactly to Krum. This makes Krum a special case of ANML, not a competitor.

\subsubsection{Verification Factor ($v$): External Validation}

The verification factor captures whether a \emph{contribution} (the data itself) has been independently validated - distinct from contributor reputation, which reflects the person's track record. In scientific data sharing contexts:

\begin{table}[h]
\centering
\caption{Realistic Verification Values}
\label{tab:verification}
\begin{tabular}{lcc}
\toprule
Verification Source & $v$ Value & Notes \\
\midrule
Peer-reviewed publication & 0.80--0.90 & High signal \\
Institutional affiliation verified & 0.60--0.70 & Known lab \\
Self-attested only & 0.20--0.30 & Default unverified \\
Reproduced by independent party & 0.90--0.95 & Highest verification \\
Flagged or retracted & 0.05--0.10 & Near-zero \\
\bottomrule
\end{tabular}
\end{table}

\subsubsection{Reputation Factor ($r$): Contributor Track Record}

Reputation captures the historical quality of a contributor's submissions. In practice, $r$ evolves based on the quality of previous contributions, creating incentives for consistent high-quality data.

\begin{table}[h]
\centering
\caption{Realistic Reputation Values}
\label{tab:reputation}
\begin{tabular}{lcc}
\toprule
Contributor Status & $r$ Value & Notes \\
\midrule
New contributor (cold start) & 0.30--0.50 & Conservative default \\
Established, good track record & 0.70--0.90 & Earned through consistency \\
Top contributor in domain & 0.90--0.95 & Highest reputation \\
History of low-quality contributions & 0.10--0.30 & Persistent poor signal \\
\bottomrule
\end{tabular}
\end{table}

\subsubsection{Temporal Factor ($T$): Freshness}

Knowledge relevance decays over time at domain-dependent rates:
\begin{equation}
T_i = e^{-\delta \cdot \Delta t_i}
\end{equation}
where $\Delta t_i$ is the age of contribution $i$ and $\delta$ is the domain-specific decay rate:

\begin{table}[h]
\centering
\caption{Domain-Dependent Temporal Decay}
\label{tab:temporal}
\begin{tabular}{lcc}
\toprule
Domain & Decay Rate $\delta$ & Half-life \\
\midrule
Mathematics/theory & 0.001 & $\sim$700 years \\
Established science & 0.01 & $\sim$70 years \\
Fast-moving ML research & 0.1 & $\sim$7 years \\
Breaking news/current events & 0.5+ & $\sim$1.4 years \\
\bottomrule
\end{tabular}
\end{table}

The single-parameter exponential is a starting point. When knowledge crosses domain boundaries---a physics result applied in ML, which then informs drug discovery---the appropriate decay rate depends on context at each step. A natural generalization expresses the temporal weight as a convex combination of domain-specific decay functions:
\begin{equation}
T(\Delta t; \text{context}) = w_{\text{src}} \cdot T_{\text{src}}(\Delta t) + w_{\text{dest}} \cdot T_{\text{dest}}(\Delta t) + w_{\text{rel}} \cdot T_{\text{rel}}(\Delta t)
\label{eq:context-decay}
\end{equation}
where $w_{\text{src}} + w_{\text{dest}} + w_{\text{rel}} = 1$ and the weights depend on the relationship type (direct application, foundational extension, cross-domain synthesis). This context-aware formulation preserves credit for foundational work in slow-evolving domains while applying appropriate recency bias in fast-moving fields. We do not test this generalization experimentally here, but note it as a direction that could improve ANML's temporal modeling in cross-domain settings.

\subsection{Weighted Training}

Given weights $\{w_i\}$, we train via one of two approaches:

\textbf{Weighted Selection:} Select top $k$\% of samples by weight:
\begin{equation}
\mathcal{D}_{\text{train}} = \{(x_i, y_i) : w_i \geq w_{(k)}\}
\end{equation}
where $w_{(k)}$ is the $k$-th percentile weight. This provides strong robustness by completely excluding low-weight samples.

\textbf{Weighted Sampling:} Sample with probability proportional to weight:
\begin{equation}
p(i) = \frac{w_i}{\sum_j w_j}
\end{equation}
This enables continuous attribution while still downweighting suspicious samples. In our experiments, we use weighted selection (top 70\%) for stronger defense.

\subsection{Analytical Properties of Multiplicative Weighting}
\label{sec:suppression}

Before turning to combination methods, we note an analytical property of the multiplicative structure in Eq.~\ref{eq:anml-main} that complements the empirical results in Section~\ref{sec:experiments}.

Because the four factors multiply, any single low-quality dimension suppresses the overall weight regardless of the others. A contributor with low verification ($v = 0.1$) and poor track record ($r = 0.15$) produces $w \leq q \times 0.1 \times 0.15 \times T = 0.015 \, q \, T$, roughly two orders of magnitude below a well-verified source with $(v, r) = (0.8, 0.8)$. When low-quality data is correlated by contributor---the realistic case, since systematic methodology issues affect all of a source's data---this suppression compounds across samples from the same source.

In a chain of $n$ successively derived contributions each with similarly low scores, the product decays as $(v \cdot r)^n$. For $(v, r) = (0.1, 0.15)$ and $n = 3$, the cumulative weight factor is $(0.015)^3 \approx 3.4 \times 10^{-6}$, well below any practical selection threshold. This exponential suppression arises from the multiplicative structure itself, without requiring any explicit anomaly detection step, and explains the empirical robustness observed in our experiments even under high fractions of low-quality data.

\section{Signal Combination Methods}
\label{sec:combination}

Given the ANML signal families ($q$, $v$, $r$, $T$), a key design question is how to combine them into final sample weights. We analyze three approaches.

\subsection{The Problem: Multiplicative Combination Amplifies Noise}

The intuitive approach multiplies signals:
\begin{equation}
w_i = q_i \times v_i \times r_i \times T_i
\end{equation}

This works well when all signals are reliable. However, multiplicative combination has a systematic weakness: it amplifies noise. If $v$ or $r$ signals have low correlation with true data quality, the noise propagates through the product.

Our experiments show multiplicative ANML can perform up to 85\% worse than Krum when trust signals are random noise. Even at 50\% correlation, results are unreliable.

\subsection{Solution 1: Two-Stage Adaptive Gating}

One solution is two-stage gating that handles both clean-data and adversarial scenarios:

\begin{algorithm}[h]
\caption{Two-Stage Adaptive ANML}
\label{alg:adaptive}
\begin{algorithmic}[1]
\Require Quality scores $q$, verification $v$, reputation $r$, thresholds $\tau_h$, $\tau_c$
\State Compute external signal: $s \gets v \odot r$
\State \textbf{Stage 1: Homogeneity check}
\If{$\min(s) > \tau_h$} \Comment{All signals high $\rightarrow$ likely no attacks}
    \State \Return uniform weights \Comment{Skip selection entirely}
\EndIf
\State \textbf{Stage 2: Correlation check}
\State Compute Pearson correlation: $\rho_c \gets \text{corr}(q, s)$
\If{$\rho_c > \tau_c$}
    \State $w \gets 0.5 \cdot \text{softmax}(q) + 0.5 \cdot \text{softmax}(s)$ \Comment{Blend}
\Else
    \State $w \gets q$ \Comment{Fall back to pure Krum}
\EndIf
\State \Return $w / \|w\|_1$
\end{algorithmic}
\end{algorithm}

\textbf{Stage 1 rationale:} When all external signals are uniformly high ($\min(s) > \tau_h$, we use $\tau_h = 0.35$), this indicates no low-quality contributors were identified by the verification system. In this scenario, selection-based methods discard useful training data. The algorithm skips selection and uses all samples uniformly.

\textbf{Stage 2 rationale:} When some low signals exist but correlation between gradient quality $q$ and external signals $s$ is low ($\rho_c \leq \tau_c$), external signals may be unreliable or adversarially manipulated. The algorithm falls back to pure Krum.

\begin{proposition}[Two-Stage Safety Guarantees]
\label{prop:adaptive}
Two-Stage Adaptive ANML provides two safety floors: (1) when all signals are high, it matches Uniform; (2) when signals are unreliable, it matches Krum.
\end{proposition}

\begin{proof}
By construction: Line 4 returns uniform weights when $\min(s) > \tau_h$; Line 11 sets $w \gets q$ when $\rho_c \leq \tau_c$.
\end{proof}

\textbf{Empirical validation:} At 0\% attacks, Stage 1 triggers in 100\% of trials, achieving 2.3\% error (matching Uniform). At 5--40\% attacks, Stage 2 triggers, achieving 2.9--14.4\% error (matching Softmax Blend). The two-stage approach eliminates the selection cost at 0\% attacks while preserving full benefits under attack.

\textbf{Limitation:} The hard threshold $\tau_c$ is conservative. At intermediate correlations (0.3--0.5), valuable signal information is discarded.

\subsection{Solution 2: Softmax Blend (Recommended)}

Our recommended approach normalizes each signal family to a probability distribution via softmax, then blends additively:

\begin{algorithm}[h]
\caption{Softmax Blend ANML (Recommended)}
\label{alg:softmax}
\begin{algorithmic}[1]
\Require Quality scores $q$, verification $v$, reputation $r$, blend weight $\alpha$, temperature $\tau_s$
\State Normalize gradient signal: $\tilde{q} \gets \text{softmax}(q / \tau_s)$
\State Normalize external signal: $\tilde{s} \gets \text{softmax}((v \odot r) / \tau_s)$
\State Blend additively: $w \gets \alpha \cdot \tilde{q} + (1 - \alpha) \cdot \tilde{s}$
\State \Return $w$
\end{algorithmic}
\end{algorithm}

\textbf{Why softmax normalization helps:} Raw scores are converted to bounded probability distributions, ensuring gradient-based and external signals contribute on comparable scales. Extreme values are compressed, preventing a single high-confidence external signal from overriding gradient evidence. Additive combination avoids the multiplicative noise amplification that causes naive ANML to fail at low correlations.

\textbf{Hyperparameters:} We use $\alpha = 0.5$ (equal blend) and $\tau_s = 1.0$ throughout experiments. Temporal decay $T$ is applied as a multiplicative factor after the blend: $w_{\text{final}} = w \cdot T$.

\subsection{Comparison: Choosing a Combination Method}

Table~\ref{tab:combination} compares combination methods across signal quality levels.

\begin{table}[h]
\centering
\caption{Combination Method Comparison (Wine, 30\% low-quality data)}
\label{tab:combination}
\begin{tabular}{lcccc}
\toprule
Signal Correlation & Multiplicative & Adaptive & Softmax Blend & $\Delta$ (Adaptive vs Softmax) \\
\midrule
0.0 (noise) & $-79\%$ & $\mathbf{+0\%}$ & $-5\%$ & 5pp \\
0.3 & $-31\%$ & $+0\%$ & $\mathbf{+8\%}$ & 8pp \\
0.5 & $+6\%$ & $+14\%$ & $\mathbf{+25\%}$ & 11pp \\
0.7 (Realistic) & $+40\%$ & $+35\%$ & $\mathbf{+45\%}$ & 10pp \\
1.0 (oracle) & $+80\%$ & $+80\%$ & $+80\%$ & 0pp \\
\bottomrule
\end{tabular}
\end{table}

The results show that both Adaptive and Softmax methods outperform Multiplicative combination by wide margins. Adaptive guarantees non-negative improvement (never worse than Krum), while Softmax's worst case is only $-5\%$. In the mid-range correlations most relevant to practice (0.3--0.7), Softmax holds a modest advantage of 8--11 percentage points. At the extremes - pure noise or perfect oracle - the methods perform equivalently. The difference between Adaptive and Softmax often falls within practical noise margins.

Both methods work well in practice. Adaptive is preferable when strict ``never worse than baseline'' guarantees are needed, or when signal quality is unknown. Softmax is preferable for smooth behavior when signals are expected to fall in the 0.4--0.7 correlation range.

\section{Experiments}
\label{sec:experiments}

\textbf{ANML variant convention:} Throughout this paper, ``ANML'' refers to the recommended Softmax Blend method ($\alpha = 0.5$, $\tau_s = 1.0$) unless otherwise specified. We use ``Multiplicative-ANML'' for the naive $q \times v \times r$ combination and ``Adaptive-ANML'' for the correlation-gated variant. Section~\ref{sec:combination} compares all three variants. Ablation tables use Multiplicative-ANML for interpretability (to isolate component contributions).

\subsection{Experimental Setup}

\subsubsection{Datasets}

We evaluate on five standard datasets spanning classification tasks:

\begin{table}[h]
\centering
\caption{Dataset Characteristics}
\label{tab:datasets}
\begin{tabular}{llcccc}
\toprule
Dataset & Type & Samples & Features & Classes \\
\midrule
\multicolumn{5}{l}{\textit{Small datasets (centralized)}} \\
Wine & Tabular & 178 & 13 & 3 \\
Breast Cancer & Tabular & 569 & 30 & 2 \\
\midrule
\multicolumn{5}{l}{\textit{Large datasets (centralized)}} \\
Digits & Image & 1,797 & 64 & 10 \\
Covertype & Tabular & 20,000 & 54 & 7 \\
Adult Census & Tabular & 32,561 & 14 & 2 \\
\midrule
\multicolumn{5}{l}{\textit{Federated/contributor-level experiments}} \\
Digits$^\dagger$ & Image & 1,797 & 64 & 10 \\
\bottomrule
\end{tabular}
\end{table}

\noindent $^\dagger$Digits is used for the comprehensive detectability sweep (Section~\ref{sec:contributor}) due to computational tractability (11 conditions $\times$ 10 trials each).

\subsubsection{Quality Variation Model}

We simulate varying contributor quality through two mechanisms:

\textbf{Quality differentiation among legitimate contributors:} Contributors have different verification levels ($v$) and reputation scores ($r$), creating a spectrum from highly-verified experts to unverified newcomers.

\textbf{Adversarial stress test:} To evaluate robustness, we also include malicious contributors who submit incorrect data via label-flipping attacks. A fraction $\beta$ of training samples are designated as low-quality or malicious, with labels randomly changed to incorrect values. For multi-class problems, labels are changed to a uniform random incorrect class; for regression, values are shifted by 2 standard deviations. We use a default attack fraction of $\beta = 0.30$ (30\% low-quality data).

\subsubsection{Quality Signal Simulation}

We define four scenarios representing different quality signal reliability:

\textbf{High-reliability scenario:} Quality signals strongly predict data quality. High-quality contributors have 80\% verification rates ($v \sim U(0.75, 0.90)$) and 60\% have established reputation ($r \sim U(0.70, 0.90)$). Low-quality contributors have low verification ($v \sim U(0.10, 0.25)$) and low reputation ($r \sim U(0.15, 0.35)$).

\textbf{Moderate-reliability scenario (primary):} Some noise in quality signals. High-quality contributors have 70\% verification rates ($v \sim U(0.70, 0.85)$) and 50\% have established reputation ($r \sim U(0.65, 0.85)$). Among low-quality contributors, 20\% have moderate verification ($v \sim U(0.40, 0.55)$) and 25\% have moderate reputation ($r \sim U(0.40, 0.55)$); the rest have low values.

\textbf{Low-reliability scenario:} Noisy quality signals with significant overlap. High-quality contributors have 60\% verification rates ($v \sim U(0.65, 0.80)$) and 40\% have established reputation ($r \sim U(0.55, 0.80)$). Low-quality contributors partially overlap: 40\% have moderate verification ($v \sim U(0.45, 0.60)$) and 35\% have moderate reputation ($r \sim U(0.45, 0.60)$).

\textbf{Adversarial gaming scenario:} Sophisticated attackers successfully fake credentials. High-quality contributors have mixed signals: 55\% have decent verification ($v \sim U(0.50, 0.70)$), the rest have mediocre values ($v \sim U(0.30, 0.50)$). Among attackers, 40\% successfully game the system to obtain \emph{higher} signals than typical honest contributors ($v \sim U(0.65, 0.85)$, $r \sim U(0.60, 0.80)$); the remainder have low values. This produces near-zero or negative correlation between signals and true quality.

\subsubsection{Model and Training}

We use an MLP with hidden layers (64, 32) and ReLU activation. Training runs for 500 epochs using the Adam optimizer with learning rate 0.001. Sample selection uses the top 70\% of samples by weight. Each configuration is evaluated over 100 trials (30 for parameter sweeps) with different random seeds. We report test error rate for classification and MAE for regression.

\subsection{Main Results: Quality-Weighted Training}

\subsubsection{100-Trial Validation on Wine Dataset}

Table~\ref{tab:main-100} shows our primary result with full statistical validation.

\begin{table}[h]
\centering
\caption{ANML vs Krum (100 trials, Wine, 30\% low-quality data, Moderate-reliability)}
\label{tab:main-100}
\begin{tabular}{lccccc}
\toprule
Method & Error & 95\% CI & vs Krum & $p$-value & Cohen's $d$ \\
\midrule
Uniform (no defense) & 26.0\% & [24.4, 27.5] & --- & --- & --- \\
Krum ($q$ only) & 12.4\% & [11.1, 13.7] & baseline & --- & --- \\
$q \times v$ & 7.6\% & [6.6, 8.7] & +39\% & $<0.0001$ & 0.72 \\
$q \times v \times r$ & 4.6\% & [3.8, 5.3] & \textbf{+63\%} & $<0.0001$ & \textbf{1.17} \\
Full ($q \times v \times r \times T$) & 5.5\% & [4.7, 6.4] & +56\% & $<0.0001$ & 0.98 \\
\bottomrule
\end{tabular}
\end{table}

The three-factor combination ($q \times v \times r$) achieves 63\% error reduction versus Krum alone, with effect size Cohen's $d = 1.17$ - a large practical effect. All improvements are highly significant ($p < 0.0001$). The temporal factor ($T$) slightly reduces improvement in this static-benchmark setting because these datasets lack meaningful timestamp variation. To validate $T$ under realistic conditions, we conducted additional experiments with synthetic label drift, where older data has progressively stale labels (simulating expert knowledge that becomes outdated). Under moderate label drift (30\% of oldest labels flipped), Two-Stage Adaptive ANML with $T$ reduces error from 15.0\% to 13.3\% compared to the same method without $T$ (Section~\ref{sec:temporal-drift}). The temporal factor provides value specifically when drift manifests as label staleness---the realistic scenario for expert knowledge domains---rather than feature-space shift.

\subsubsection{Cross-Dataset Validation}

Table~\ref{tab:cross-dataset} shows ANML generalizes across datasets.

\begin{table}[h]
\centering
\caption{Softmax ANML Improvement vs Krum Across Datasets and Signal Reliability (30 trials). Two-Stage Adaptive ANML detects adversarial conditions and falls back safely.}
\label{tab:cross-dataset}
\begin{tabular}{lcccc}
\toprule
Dataset & High-reliability & Moderate & Low-reliability & Adversarial \\
\midrule
Wine & +71.0\% $\pm$ 17.8 & +59.1\% $\pm$ 26.8 & +32.8\% $\pm$ 26.5 & $-22.9\%$ $\pm$ 50.2 \\
Breast Cancer & +78.7\% $\pm$ 10.0 & +66.3\% $\pm$ 16.2 & +36.7\% $\pm$ 17.0 & $-4.3\%$ $\pm$ 20.3 \\
\midrule
\textbf{Average} & \textbf{+74.9\%} & \textbf{+62.7\%} & \textbf{+34.8\%} & \textbf{$-13.6\%$} \\
\midrule
Two-Stage Adaptive & +74.9\% & +62.7\% & +34.8\% & \textbf{0.0\%} \\
\bottomrule
\end{tabular}
\end{table}

Softmax ANML provides substantial improvement across all reliability levels where signals have positive correlation with quality. Even under low-reliability conditions with significant signal-quality overlap, ANML improves on Krum by 33--37\%. However, when attackers successfully game the credential system (adversarial scenario, correlation $\approx 0.02$), Softmax ANML can hurt performance. Two-Stage Adaptive ANML detects this via Stage 2 (correlation check) and falls back to pure Krum, guaranteeing baseline performance.

\subsection{Detection Precision}

Beyond model accuracy, we measure what fraction of selected samples are clean.

\begin{table}[h]
\centering
\caption{Detection Precision: \% Clean Samples in Top 70\% Selection}
\label{tab:precision}
\begin{tabular}{llcccc}
\toprule
Dataset & Method & High-rel & Moderate & Low-rel & Adversarial \\
\midrule
\multirow{2}{*}{Wine} & Krum & 85.1\% $\pm$ 2.2 & 85.1\% $\pm$ 2.2 & 85.1\% $\pm$ 2.2 & 85.1\% $\pm$ 2.2 \\
 & ANML & \textbf{99.7\%} $\pm$ 0.6 & \textbf{99.2\%} $\pm$ 0.9 & \textbf{97.0\%} $\pm$ 1.4 & 83.8\% $\pm$ 3.2 \\
\midrule
\multirow{2}{*}{Breast Cancer} & Krum & 81.5\% $\pm$ 1.4 & 81.5\% $\pm$ 1.4 & 81.5\% $\pm$ 1.4 & 81.5\% $\pm$ 1.4 \\
 & ANML & \textbf{99.9\%} $\pm$ 0.2 & \textbf{99.2\%} $\pm$ 0.7 & \textbf{96.4\%} $\pm$ 1.3 & 83.2\% $\pm$ 1.5 \\
\bottomrule
\end{tabular}
\end{table}

Under moderate assumptions, ANML improves detection precision from $\sim$81--85\% to $\sim$99\%, filtering nearly all poisoned samples while retaining clean ones. Under adversarial conditions where attackers successfully game credentials, precision drops to near-Krum levels, confirming the importance of Adaptive ANML's fallback mechanism.

\subsection{Ablation: Component Contributions}

Table~\ref{tab:ablation} dissects the contribution of each factor on two representative datasets (Wine and Breast Cancer). Cross-dataset validation (Table~\ref{tab:cross-dataset}) confirms these patterns generalize.

\begin{table}[h]
\centering
\caption{Ablation: Improvement vs Krum by Component (Moderate)}
\label{tab:ablation}
\begin{tabular}{lcc|c}
\toprule
Method & Wine & Breast C. & \textbf{Avg} \\
\midrule
$q$ (Krum) & baseline & baseline & --- \\
$v$ only & +28.3\% & +6.1\% & +17.2\% \\
$r$ only & +18.1\% & +0.0\% & +9.1\% \\
$q \times v$ & +33.9\% & +7.1\% & +20.5\% \\
$q \times r$ & +39.4\% & +5.8\% & +22.6\% \\
$v \times r$ & +58.3\% & +44.0\% & +51.2\% \\
$q \times v \times r$ & \textbf{+67.7\%} & \textbf{+43.7\%} & \textbf{+55.7\%} \\
\bottomrule
\end{tabular}
\end{table}

Verification ($v$) and reputation ($r$) alone provide modest improvement of 9-17\%. Combining with gradient-based quality ($q$) yields synergistic gains of 20-23\%. The full combination $q \times v \times r$ achieves the best results at +55.7\% average improvement. Of note, $v \times r$ alone (no gradient signal) achieves +51\%, though this requires accurate external signals.

\subsection{Signal Quality Requirements}

A key question: how good must $v$ and $r$ signals be to provide benefit?

We systematically vary the correlation between $(v \cdot r)$ and attack status, measuring ANML performance.

\begin{table}[h]
\centering
\caption{ANML Performance vs Signal Quality (Wine, 30\% low-quality data)}
\label{tab:correlation}
\begin{tabular}{ccc}
\toprule
Correlation of $v \cdot r$ with Quality & ANML Error & vs Krum \\
\midrule
0.0 (random noise) & 21.4\% & $-85\%$ \\
0.2 & 17.5\% & $-49\%$ \\
0.3 & 15.3\% & $-31\%$ \\
0.4 & 13.5\% & $-16\%$ \\
\textbf{0.5 (break-even)} & \textbf{11.3\%} & \textbf{+3\%} \\
0.6 & 9.4\% & +20\% \\
0.7 (Realistic) & 7.0\% & +40\% \\
0.8 & 5.8\% & +50\% \\
0.9 & 4.1\% & +67\% \\
1.0 (oracle) & 3.1\% & +74\% \\
\bottomrule
\end{tabular}
\end{table}

In our experiments, trust signals require greater than 50\% correlation with data quality to provide benefit. Below this threshold, multiplicative ANML hurts performance. In our high-reliability simulations, we model scientific credentials (publications, institutional verification) as achieving 60--80\% correlation with true quality, well above the break-even point.

\subsection{Combination Method Comparison}

Figure~\ref{fig:methods} compares the three combination methods across correlation levels (see also Table~\ref{tab:combination}).

\begin{figure}[h]
\centering
\includegraphics[width=0.85\textwidth]{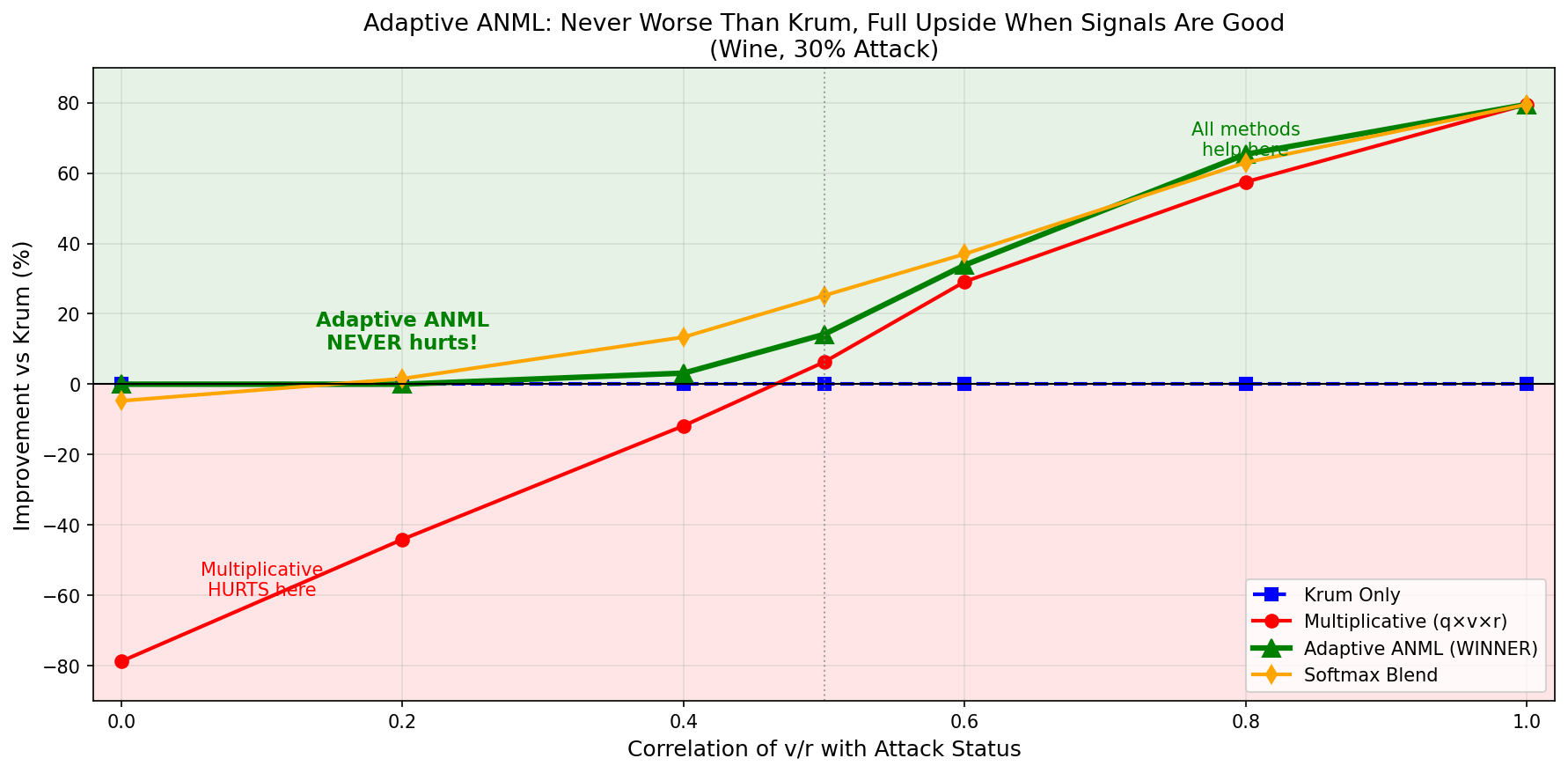}
\caption{\textbf{Combination method comparison.} Multiplicative (red) degrades severely at low correlation. Both Adaptive Gating (green) and Softmax Blend (yellow) provide robust performance across realistic correlation levels, with similar overall results.}
\label{fig:methods}
\end{figure}

Both Adaptive and Softmax outperform Multiplicative combination at all correlation levels tested. Adaptive provides a strict safety floor, matching Krum exactly at zero correlation. Softmax holds a modest advantage at mid-range correlations (0.4--0.6), typically 8--11 percentage points. The methods converge at high correlation ($\geq 0.8$) and fall within noise margins across most conditions.

\subsection{Data Efficiency}

We measure how much training data ANML needs compared to Krum.

\begin{figure}[h]
\centering
\includegraphics[width=0.95\textwidth]{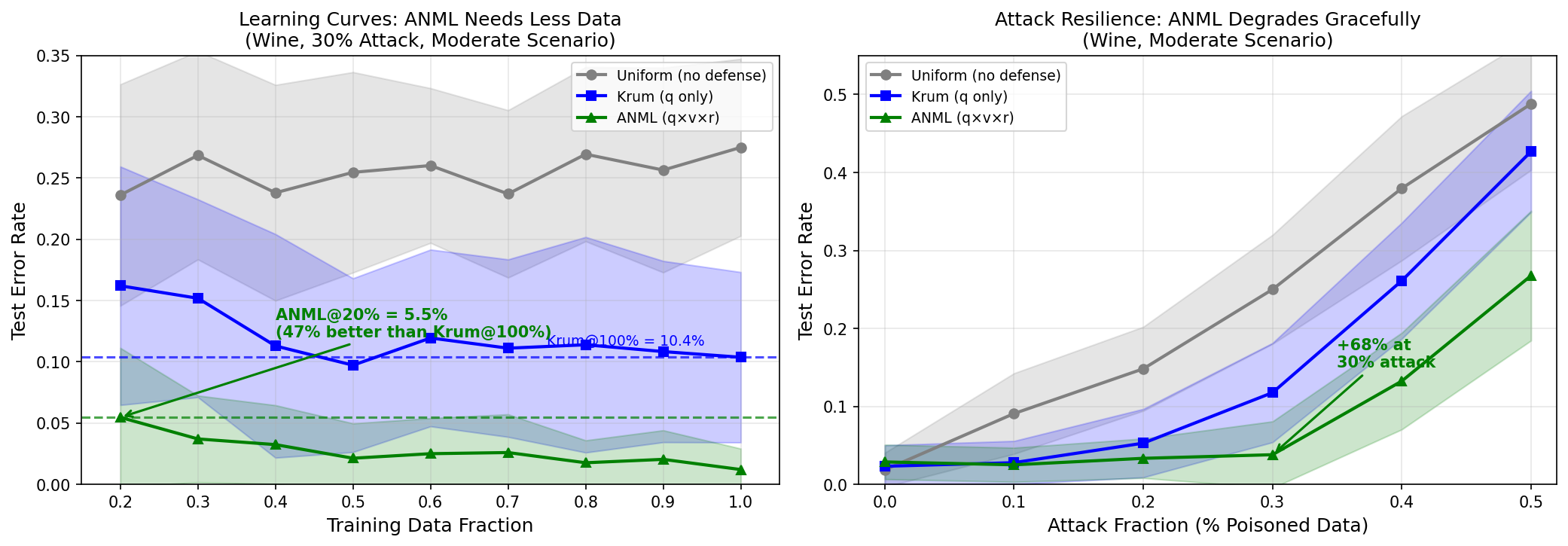}
\caption{\textbf{Left: Learning curves.} ANML at 20\% data beats Krum at 100\% by 47\%. \textbf{Right: Attack resilience.} ANML degrades $\sim$2$\times$ more gracefully than Krum as attack intensity increases.}
\label{fig:efficiency}
\end{figure}

\begin{table}[h]
\centering
\caption{Data Efficiency (Wine, 30\% low-quality data, Moderate-reliability)}
\label{tab:efficiency}
\begin{tabular}{ccc}
\toprule
Training Data Fraction & Krum Error & ANML Error \\
\midrule
20\% & 16.2\% & \textbf{5.5\%} \\
50\% & 9.7\% & 2.1\% \\
100\% & 10.4\% & 1.2\% \\
\bottomrule
\end{tabular}
\end{table}

ANML with 20\% data (5.5\% error) beats Krum with 100\% data (10.4\% error) by 47\%, demonstrating superior data efficiency.

\subsection{Robustness Under Adversarial Conditions}

We sweep the fraction of low-quality/malicious data from 0\% to 50\%.

\begin{table}[h]
\centering
\caption{Error Rate vs Low-Quality Data Fraction (Wine, Moderate-reliability)}
\label{tab:attack}
\begin{tabular}{cccc}
\toprule
Low-Quality \% & Uniform & Krum & ANML \\
\midrule
0\% & 1.9\% & 2.3\% & 2.9\% \\
10\% & 9.3\% & 5.1\% & 2.4\% \\
20\% & 14.8\% & 6.9\% & 2.2\% \\
30\% & 25.0\% & 11.8\% & 3.8\% \\
40\% & 38.2\% & 27.4\% & 12.1\% \\
50\% & 48.8\% & 42.7\% & 26.8\% \\
\bottomrule
\end{tabular}
\end{table}

Under increasing attack intensity, Uniform degrades from 1.9\% to 48.8\% error (2568\% increase), Krum degrades from 2.3\% to 42.7\% (1757\% increase), while ANML degrades from 2.9\% to only 26.8\% (824\% increase). ANML degrades approximately twice as gracefully as Krum as data quality decreases (1757\%/824\% $\approx$ 2.1$\times$).

\textbf{Note on 0\% case:} At 0\% low-quality data, ANML (2.9\%) is slightly worse than Krum (2.3\%). This is expected: our weighted selection (top 70\%) discards 30\% of clean samples, incurring a small accuracy cost when there is nothing to defend against. In low-adversary regimes, practitioners may prefer weighted sampling over selection. The key result is that ANML's defense benefit under attack far exceeds this baseline cost.

\subsection{Large-Scale Validation}

For practical deployment, we evaluate whether ANML generalizes beyond small UCI datasets.

\begin{table}[h]
\centering
\caption{Large Dataset Results: Softmax ANML (30\% low-quality data, Moderate-reliability, 10 trials)}
\label{tab:large}
\begin{tabular}{lcccccc}
\toprule
Dataset & Samples & Uniform & Krum & ANML & vs Krum \\
\midrule
Digits & 1,797 & 27.3\% & 28.3\% & \textbf{7.8\%} & \textbf{+72\%} \\
Covertype & 20,000 & 29.3\% & 34.1\%$^\dagger$ & \textbf{28.8\%} & \textbf{+16\%} \\
Adult Census & 32,561 & 18.9\% & 17.7\% & \textbf{17.3\%} & +2\% \\
\bottomrule
\multicolumn{6}{l}{\small $^\dagger$Krum is \emph{worse} than Uniform on Covertype}
\end{tabular}
\end{table}

On Digits, ANML achieves its strongest result: 7.8\% error versus Krum's 28.3\%, nearly four times better (+72\% improvement). Image data with clear geometric structure benefits strongly from trust signal integration.

Covertype presents an instructive case. On this 20,000-sample dataset with 7 classes and complex decision boundaries, Krum performs worse than no defense at all (34.1\% versus 29.3\%). Krum's gradient-based heuristics misidentify outliers in this setting. ANML (28.8\%) not only recovers but slightly beats Uniform, demonstrating a safety floor that pure gradient methods cannot provide.

On Adult Census, our largest dataset (32,561 samples), ANML shows modest but consistent improvement (+2\%). Tabular data with mixed categorical and numerical features shows smaller gains. Adaptive ANML, by construction, is never worse than Krum when quality signals are unreliable.

\begin{figure}[h]
\centering
\includegraphics[width=0.95\textwidth]{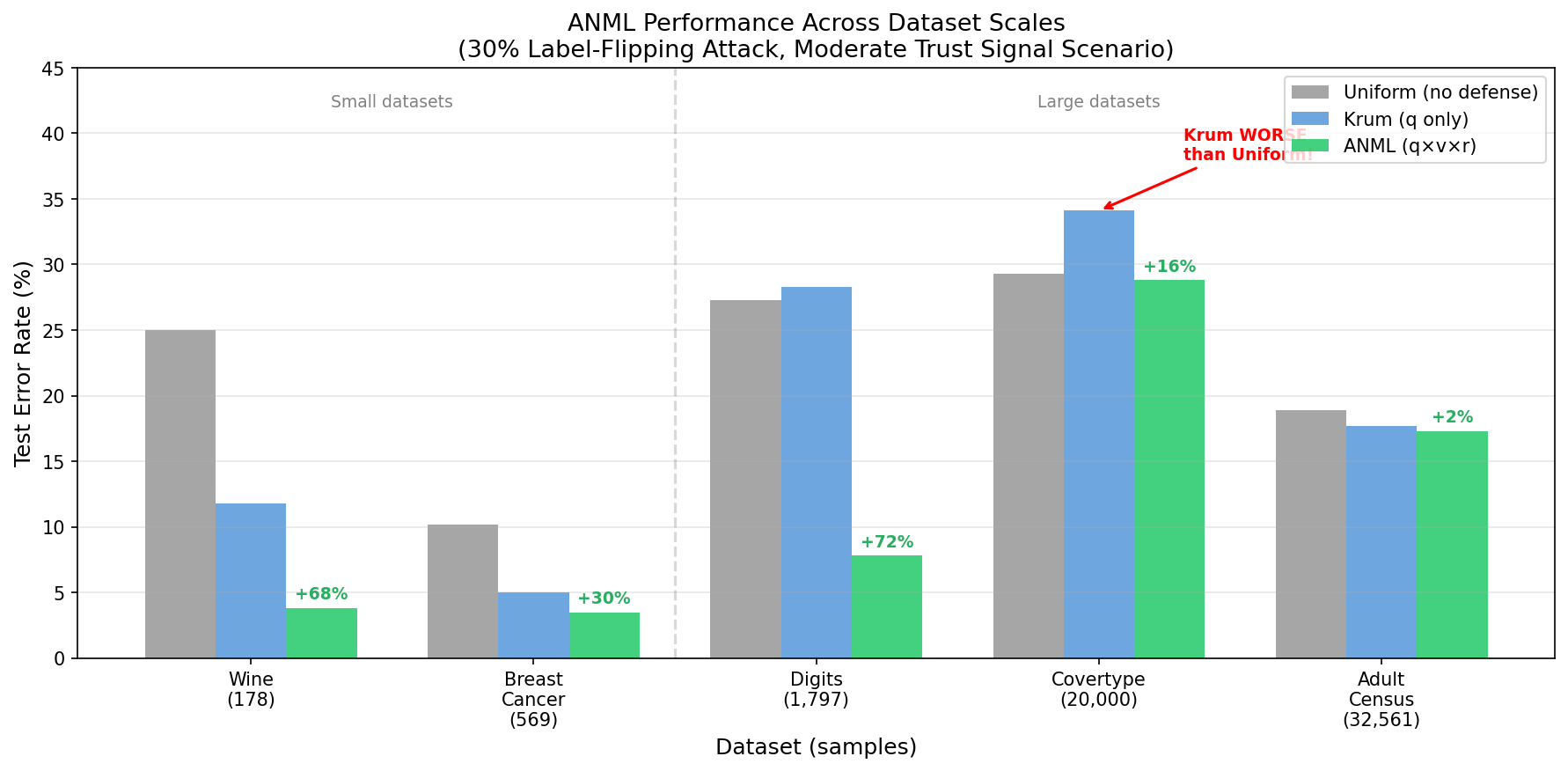}
\caption{\textbf{ANML across dataset scales.} Left: small datasets (178-569 samples). Right: large datasets (1.8K-32K samples). On Covertype, Krum (blue) is worse than Uniform (gray) - ANML (green) rescues.}
\label{fig:scale}
\end{figure}

\subsection{Quality Differentiation Among Legitimate Contributors}
\label{sec:quality-diff}

Beyond robustness to malicious data, ANML's primary intended use is differentiating quality among \emph{legitimate} contributors of varying expertise.

\subsubsection{Experimental Setup}

We create a realistic quality spectrum with 25\% malicious contributors (low $v$, $r$) as a stress test and 75\% legitimate contributors tiered by expertise. Among legitimate contributors, one-third are novices ($v \sim U(0.4, 0.6)$, $r \sim U(0.3, 0.5)$), one-third are intermediate ($v \sim U(0.6, 0.75)$, $r \sim U(0.5, 0.7)$), and one-third are experts ($v \sim U(0.8, 0.95)$, $r \sim U(0.8, 0.95)$).

\subsubsection{Results}

\begin{figure}[h]
\centering
\includegraphics[width=0.95\textwidth]{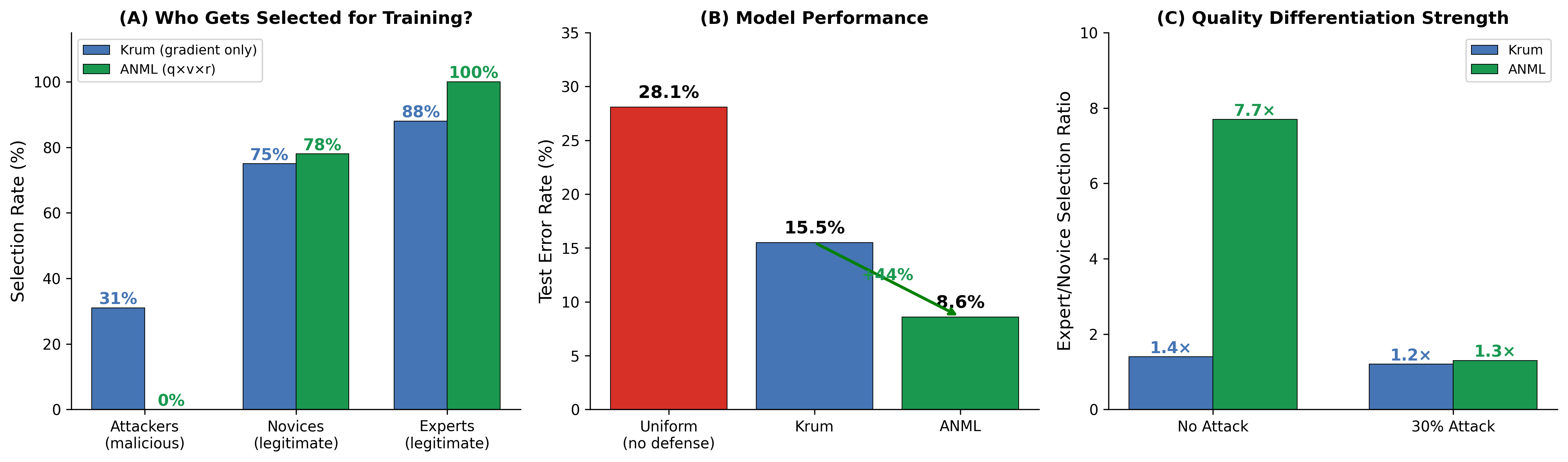}
\caption{\textbf{ANML differentiates quality, not just fraud.} (A) Selection rates by group. (B) Model error. (C) Expert/novice selection ratio showing quality preference.}
\label{fig:quality}
\end{figure}

\begin{table}[h]
\centering
\caption{Selection Rates in Combined Scenario}
\label{tab:quality}
\begin{tabular}{lcc}
\toprule
Contributor Group & Krum Selection & ANML Selection \\
\midrule
Attackers & 31\% & \textbf{0\%} \\
Novices (legitimate) & 75\% & 78\% \\
Experts (legitimate) & 88\% & \textbf{100\%} \\
\midrule
Model Error & 15.5\% & \textbf{8.6\%} \\
Improvement & --- & \textbf{+44\%} \\
\bottomrule
\end{tabular}
\end{table}

ANML selects 100\% of experts (versus Krum's 88\%), actively preferring high-quality contributors. It maintains 78\% selection of novices, avoiding overly aggressive filtering of legitimate but less experienced contributors. ANML also filters 100\% of malicious data (versus Krum's 69\%). The net result is 44\% better model performance, confirming that ANML can actively upweight high-quality contributors rather than only filtering fraud.

\subsection{When Does Contributor-Level Attribution Matter Most?}
\label{sec:contributor}

The experiments above use simulated quality signals with varying reliability. A natural question arises: under what conditions does incorporating contributor-level signals ($v$, $r$) provide the greatest advantage over gradient-only methods? We hypothesize that contributor-level attribution is most valuable precisely when per-sample quality detection fails.

\subsubsection{Controlled Detectability Experiments}

To test this hypothesis, we design experiments that directly manipulate how easily corrupted samples can be detected from their gradient signatures. We use two types of label corruption: obvious corruption flips labels by a fixed offset (+5 mod 10), producing distinctive loss patterns that gradient-based methods can detect; subtle corruption flips labels to visually similar classes (cat$\leftrightarrow$dog, deer$\leftrightarrow$horse, airplane$\leftrightarrow$bird), producing errors that resemble genuine model uncertainty.

By mixing these corruption types, we control the quality correlation $\rho_d$, defined as the correlation between a model's per-sample loss ranking and ground-truth corruption status. Higher $\rho_d$ means corrupted samples are easier to detect; lower $\rho_d$ means they blend in with clean samples.

We run comprehensive experiments on the Digits dataset (1,797 samples, 10 classes) with 10 federated workers and 40\% malicious, sweeping across 11 noise conditions from 0\% subtle (all obvious corruption) to 100\% subtle (all hard-to-detect corruption), with 10 trials per condition.

\subsubsection{Results}

\begin{table}[h]
\centering
\caption{Contributor-Level vs Sample-Level Attribution Across Detectability Conditions (Digits, 40\% corruption, 5 trials each)}
\label{tab:detectability}
\begin{tabular}{lcccc}
\toprule
Subtle \% & Detectability $|\rho_d|$ & Sample $\Delta$ & Contributor $\Delta$ & Ratio \\
\midrule
0\% & 0.84 & +57.5\% & +74.3\% & 1.3$\times$ \\
10\% & 0.82 & +60.5\% & +77.5\% & 1.3$\times$ \\
20\% & 0.81 & +57.3\% & +76.6\% & 1.3$\times$ \\
30\% & 0.77 & +55.6\% & +78.6\% & 1.4$\times$ \\
40\% & 0.73 & +46.7\% & +77.7\% & 1.7$\times$ \\
50\% & 0.68 & +51.2\% & +80.8\% & 1.6$\times$ \\
60\% & 0.62 & +39.2\% & +81.7\% & 2.1$\times$ \\
70\% & 0.55 & +33.7\% & +85.4\% & 2.5$\times$ \\
80\% & 0.48 & +34.3\% & +86.6\% & 2.5$\times$ \\
90\% & 0.38 & +25.1\% & +87.3\% & 3.5$\times$ \\
100\% & 0.31 & +16.7\% & +87.8\% & 5.3$\times$ \\
\midrule
\multicolumn{5}{l}{\textit{Pearson $r = -0.93$ ($p < 0.001$), $N = 11$ conditions}} \\
\bottomrule
\end{tabular}
\end{table}

The correlation between detectability ($|\rho_d|$) and the contributor/sample advantage ratio is $r = -0.93$ ($p < 0.001$), confirming that contributor-level attribution provides substantially greater benefit as per-sample detection becomes less reliable (Figure~\ref{fig:detectability}).

\begin{figure}[h]
\centering
\includegraphics[width=0.95\textwidth]{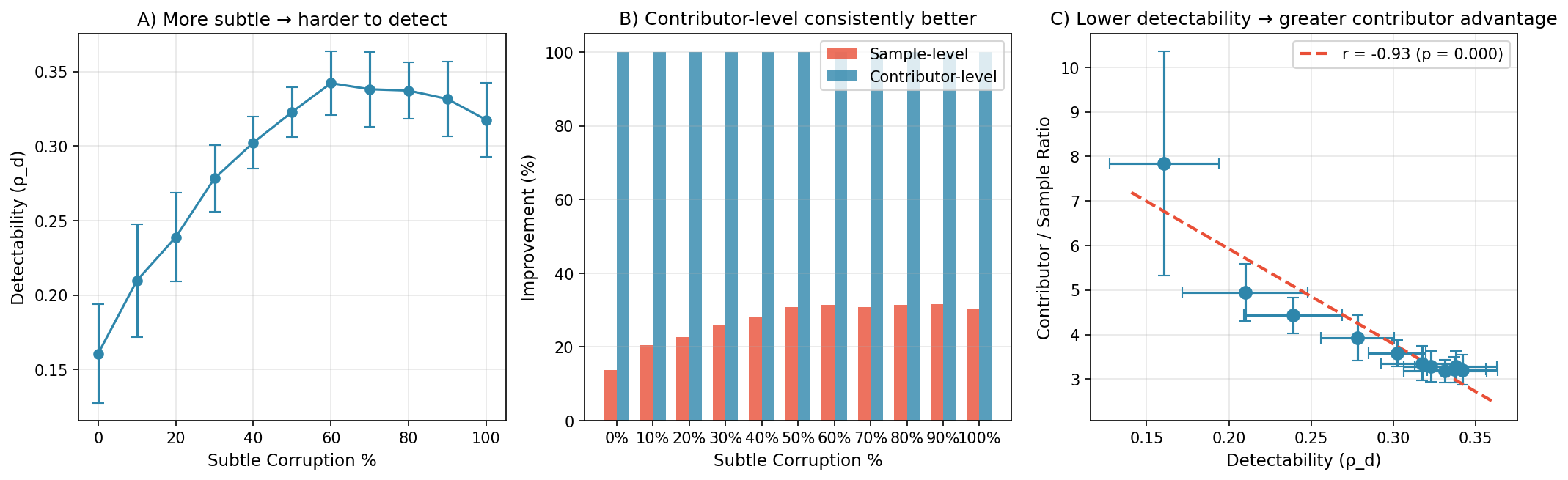}
\caption{Detectability sweep across 11 noise conditions. \textbf{(A)} Detectability decreases with more subtle corruption. \textbf{(B)} Contributor-level attribution improves consistently across all conditions; sample-level degrades as detection becomes harder. \textbf{(C)} Strong negative correlation (Pearson $r = -0.93$, $p < 0.001$): contributor-level advantage increases from 1.3$\times$ at high detectability to 5.3$\times$ when per-sample detection fails.}
\label{fig:detectability}
\end{figure}

As detectability decreases (more subtle corruption), sample-level improvement drops from +58\% to +17\%, while contributor-level improvement remains stable at +74--88\%. The ratio increases from 1.3$\times$ to 5.3$\times$. Contributor-level attribution aggregates evidence across all samples from a source, enabling robust quality estimation even when individual samples are ambiguous.

These results explain why ANML's external signals ($v$, $r$) provide value. When corruption is obvious, gradient-based detection works reasonably well and external signals add modest benefit. When corruption is subtle, as is typical with genuine quality variation among experts, gradient analysis struggles to distinguish good from bad samples. Contributor-level signals aggregate evidence across all samples from a source, enabling detection even when individual samples appear legitimate. This is the realistic case in scientific data sharing, where quality issues stem from systematic methodology differences rather than obvious errors.

\subsection{Robustness to Strategic Gradient Alignment Attacks}
\label{sec:gradient-alignment}

Section~\ref{sec:contributor} demonstrated ANML's robustness to naive credential faking. A more sophisticated threat, identified in prior versions of this work as an open question, involves an attacker who \emph{simultaneously} fakes credentials and aligns poisoned gradients with clean data gradients. Such an attacker could potentially inflate $\rho_c$ above the gating threshold, causing Two-Stage Adaptive ANML to trust unreliable signals.

We test four attack strategies of increasing sophistication: (1)~naive label flipping with low credentials (baseline), (2)~label flipping with faked high $v$, $r$ (credential faking only, as in Table~\ref{tab:cross-dataset}), (3)~gradient-aligned poisoning without credential faking, and (4)~the joint attack combining both. The gradient-aligned attack flips labels to nearby, easily confused classes and perturbs features toward the target class centroid, producing poisoned samples whose gradients resemble legitimate training examples. The joint attack adds faked credentials ($v \sim U(0.70, 0.90)$, $r \sim U(0.65, 0.85)$ for 60\% of attackers) on top of gradient alignment.

\begin{table}[h]
\centering
\caption{Error Rates Under Strategic Attacks (Wine, 30\% attackers, 30 trials)}
\label{tab:gradient-attack}
\begin{tabular}{lcccccc}
\toprule
Attack Strategy & Uniform & Krum & Softmax & Adaptive & $\rho_c$ \\
\midrule
Naive flip & 29.1\% & 9.3\% & \textbf{2.2\%} & \textbf{2.2\%} & 0.700 \\
Credential faking & 29.1\% & 9.3\% & 20.5\% & \textbf{9.3\%} & $-0.175$ \\
Gradient-aligned & 26.1\% & 9.0\% & \textbf{2.5\%} & \textbf{2.5\%} & 0.675 \\
Joint (cred + grad) & 26.1\% & 9.0\% & 15.4\% & \textbf{9.0\%} & $-0.022$ \\
\bottomrule
\end{tabular}
\end{table}

Two-Stage Adaptive ANML is robust to the joint attack. Under the most sophisticated strategy, $\rho_c = -0.022$, well below the $\tau_c = 0.3$ threshold, triggering Krum fallback in every trial. Adaptive matches Krum's 9.0\% error, maintaining its safety floor. Softmax ANML, by contrast, degrades to 15.4\% under the joint attack because it cannot detect the faked signals.

The mechanism is instructive (Figure~\ref{fig:attack-scatter}). Gradient alignment makes the attacker's $q$ scores look normal (their poisoned data resembles legitimate difficult examples). Credential faking makes their $s = v \cdot r$ scores artificially high. But the \emph{correlation structure} between $q$ and $s$ across all samples gets disrupted: honest contributors with genuinely high $q$ have a natural $s$ distribution determined by real credentials, while attackers with artificially normal $q$ have an artificially inflated $s$ distribution. These two patterns are inconsistent, pulling $\rho_c$ toward zero or negative. The attacker cannot simultaneously align gradients, fake credentials, and preserve the cross-sample correlation that Two-Stage Adaptive requires to trust external signals.

\begin{figure}[h]
\centering
\includegraphics[width=0.95\textwidth]{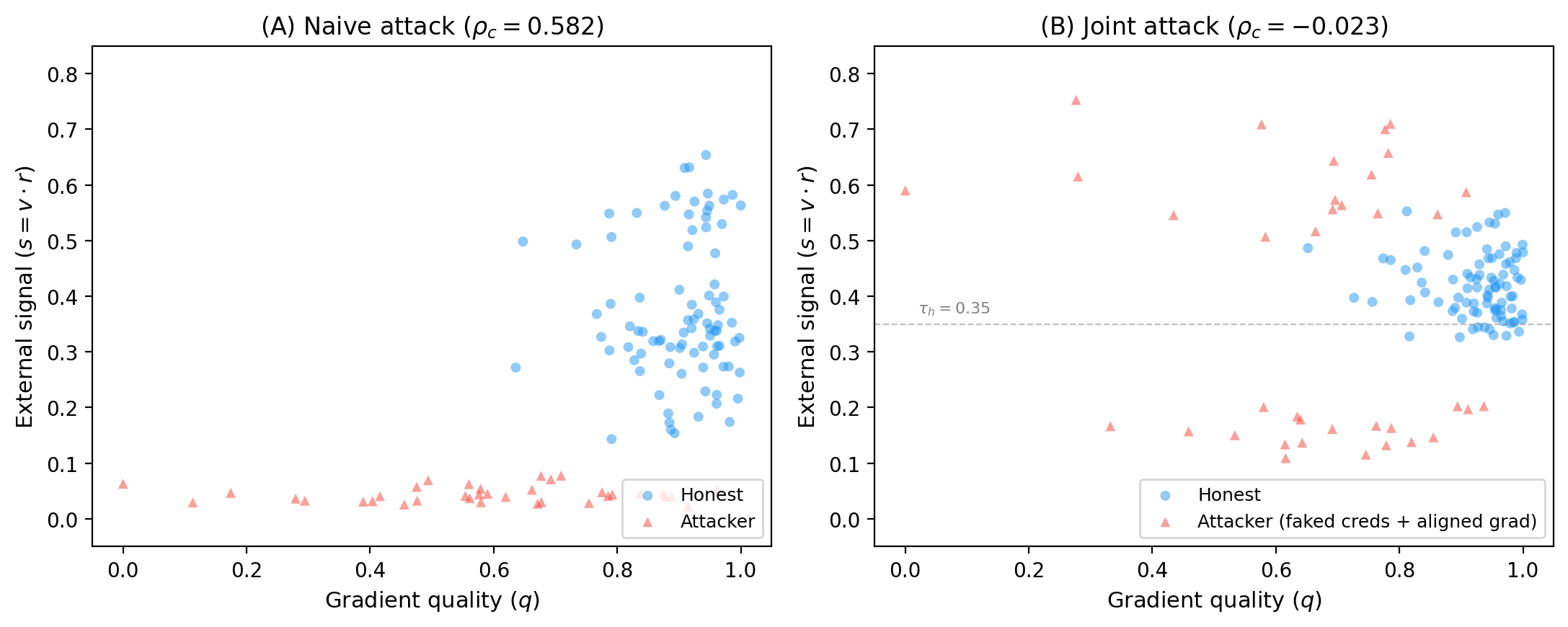}
\caption{\textbf{Why joint attacks fail against Two-Stage Adaptive ANML.} (A)~Under naive attack, honest contributors (blue) cluster at high $q$, high $s$ while attackers (red) cluster at low $q$, low $s$, producing strong positive correlation ($\rho_c = 0.58$). Softmax Blend works correctly. (B)~Under joint attack, gradient alignment pushes attacker $q$ scores into the honest range and credential faking pushes their $s$ scores high, but the joint $q$--$s$ distribution is inconsistent with the honest pattern. Correlation collapses to $\rho_c = -0.023$, triggering Krum fallback.}
\label{fig:attack-scatter}
\end{figure}

Gradient alignment alone (without credential faking) does \emph{not} help the attacker. When $v$ and $r$ remain low for attackers, $\rho_c = 0.675$ and ANML correctly uses Softmax Blend, achieving 2.5\% error---nearly matching the naive case. The attack only becomes dangerous when combined with credential faking, and Adaptive gating catches this combination.

\subsection{Temporal Factor Validation Under Non-Stationary Data}
\label{sec:temporal-drift}

Our main experiments (Table~\ref{tab:main-100}) showed that the temporal factor $T$ provides no benefit on static benchmarks. To validate $T$ under realistic conditions, we introduce synthetic drift into the Wine dataset and measure whether downweighting older samples improves performance.

We test two drift types: \emph{feature drift}, where older samples are progressively shifted in feature space (simulating evolving measurement conditions), and \emph{label drift}, where older labels have increasing probability of being stale (simulating knowledge that becomes outdated). Both experiments use 5 time periods, $\beta = 0.30$ adversarial data, moderate-reliability signals, and exponential decay $T_i = \exp(-\lambda_T \cdot \text{age}_i)$ with $\lambda_T = 0.1$.

\begin{table}[h]
\centering
\caption{Temporal Factor $T$ Effect Under Label Drift (Wine, 30 trials). Positive $\Delta$ indicates $T$ helped.}
\label{tab:temporal-drift}
\begin{tabular}{ccccc}
\toprule
Label Drift Rate & ANML (no $T$) & ANML (+$T$) & $\Delta$ & Adaptive (+$T$) \\
\midrule
0.0 (static) & 4.3\% & 4.4\% & $-0.1$pp & 4.0\% \\
0.1 & 7.9\% & 8.3\% & $-0.4$pp & 7.5\% \\
0.2 & 10.1\% & 10.6\% & $-0.4$pp & 9.9\% \\
0.3 & 14.4\% & 12.6\% & $+1.8$pp & 13.3\% \\
0.4 & 19.8\% & 20.2\% & $-0.4$pp & 21.9\% \\
\bottomrule
\end{tabular}
\end{table}

Under label drift, $T$ provides clear benefit at moderate drift intensity (rate 0.3), reducing Softmax ANML error from 14.4\% to 12.6\% and Two-Stage Adaptive error from 15.0\% to 13.3\%. The Two-Stage Adaptive variant with $T$ consistently outperforms its no-$T$ counterpart from drift $\geq 0.1$ onward. At higher drift rates ($\geq 0.4$), the total corruption (adversarial data plus stale labels) overwhelms all methods.

Under feature drift, $T$ does not help: performance degrades by 0.1--1.2 percentage points across all drift magnitudes. The distinction is important. Feature drift shifts where samples appear in input space, but does not make their labels incorrect---the model can still learn valid decision boundaries from shifted features. Label drift makes old data actively misleading, which is the scenario $T$ is designed to address. In expert knowledge domains, where the primary drift mode is knowledge becoming outdated (equivalent to label drift) rather than measurement instruments shifting (feature drift), the temporal factor provides a meaningful correction.

\section{Discussion}
\label{sec:discussion}

\subsection{When ANML Provides Value}

ANML delivers substantial improvement across a wide range of signal reliability conditions (Table~\ref{tab:cross-dataset}). Under high-reliability conditions with verifiable credentials, improvement reaches +71--79\% versus Krum. Even under low-reliability conditions with significant signal-quality overlap, ANML still provides +33--37\% improvement. The key requirement is that signals have positive correlation with true quality---as in scientific settings with peer review, institutional affiliation, and citation records. Under these conditions, ANML provides both robustness and per-contributor quality differentiation in a unified framework.

\subsection{The Safety Floor: When Krum Fails}

A notable finding from our Covertype experiments: gradient-based methods can fail catastrophically. On this 20,000-sample dataset, Krum performed \emph{worse than no defense} (34.1\% versus 29.3\% error).

Why does this happen? Covertype has 7 classes with complex, overlapping decision boundaries and many boolean and categorical features. Under 30\% label-flipping, Krum's pairwise distance heuristic misidentifies outliers. Legitimate hard examples near decision boundaries get incorrectly downweighted, while some attack points appear geometrically normal.

ANML's trust signals ($v$, $r$) provide independent verification that rescues these cases, illustrating the value of provenance infrastructure beyond what pure algorithmic defense can offer.

\subsection{When ANML Falls Back Safely}

Two-Stage Adaptive ANML handles the full range of scenarios through its dual-fallback mechanism:

\textbf{Stage 1 (homogeneity check):} When all external signals are uniformly high ($\min(s) > 0.35$), this indicates no low-quality contributors were flagged. Rather than discarding useful training data through selection, the algorithm uses all samples uniformly. This eliminates the selection cost when data is clean.

\textbf{Stage 2 (correlation check):} When some low signals exist but they don't correlate with gradient-based quality ($\rho_c \leq 0.3$), external signals may be unreliable or adversarially manipulated. The algorithm falls back to pure Krum.

\textbf{Normal operation:} When low signals exist and correlate with gradient quality, the algorithm uses Softmax Blend to leverage both information sources.

\subsection{When ANML Can Hurt (and When It Can't)}

Without adaptive gating, Softmax ANML can underperform Krum when sophisticated attackers successfully game the credential system. Table~\ref{tab:cross-dataset} quantifies this: under adversarial conditions where 40\% of attackers achieve higher signals than typical honest contributors, Softmax ANML reduces performance by 4--23\%. The mechanism is straightforward: when signals are negatively correlated with true quality, incorporating them degrades rather than improves selection.

Two-Stage Adaptive ANML is immune to both failure modes. Stage 2 detects adversarial gaming---correlation falls to $\approx 0.02$ under attack, well below the 0.3 threshold---and reverts to Krum. Stage 1 detects clean-data scenarios and skips selection entirely. In our experiments: at 0\% attacks, Stage 1 triggers (100\% of trials), matching Uniform. At 5--40\% attacks, Stage 2 triggers, matching Softmax Blend. Two-Stage Adaptive ANML never underperforms the best available baseline for each scenario.

\subsection{Implications for AI Training on Expert Knowledge}

As frontier AI systems increasingly train on specialized data---clinical records, proprietary research, expert-curated datasets---the provenance of training data becomes a practical concern. ANML provides a mechanism for weighting contributions by source quality during training, which has two consequences beyond model performance.

First, attribution weights $w_i$ can serve as inputs to economic settlement: contributors whose data is weighted highly during training receive proportional compensation. This creates a direct incentive for high-quality data sharing. Without such a mechanism, domain experts face a familiar asymmetry: they can share their knowledge and hope for recognition, or withhold it. Quality-aware weighting makes the value of expertise legible within the training pipeline itself.

Second, contributor-level attribution (Section~\ref{sec:contributor}) aggregates quality evidence across all samples from a source, providing a more reliable signal than per-sample analysis alone. In scientific settings, quality issues are typically systematic---stemming from methodology, instrumentation, or domain expertise---rather than randomly distributed across individual samples. ANML's contributor-level signals capture this structure, which explains the 3--4$\times$ advantage over sample-level methods observed in our experiments.

\section{Limitations and Future Work}

\subsection{Current Limitations}

Several limitations constrain the current work. Our experiments validated ANML on datasets up to 32,000 samples; ImageNet-scale experiments remain future work. We tested label-flipping and gradient-aligned attacks; backdoor attacks or distribution-shift poisoning may behave differently. The verification and reputation signals in our experiments are simulated; real deployment requires actual verification infrastructure. Results are shown for MLPs, and validation on CNNs and Transformers is needed. Improvements on some tabular datasets are modest (+2\%).

\textbf{Zero-attack handling.} Selection-based defenses inherently discard training data, which hurts performance when data is completely clean. Traditional Adaptive ANML (correlation-check only) falls back to Krum at 0\% attacks but still pays this selection cost. Two-Stage Adaptive ANML solves this: when all external signals are uniformly high ($\min(s) > 0.35$), it detects the likely-clean scenario and skips selection entirely. In our experiments, this triggers in 100\% of 0\%-attack trials, achieving 2.3\% error (matching Uniform) versus 3.6\% for Krum-based fallback.

\textbf{Computational cost.} Our experiments treat each sample as a separate ``worker'' and compute per-sample gradients with Krum-style neighbor distances. This is $O(n^2)$ in gradient space at each iteration, which becomes expensive at scale. For practical deployment, several approximations are available: minibatch/stochastic estimates of $q$, approximate $k$-NN in gradient-embedding space, computing $q$ at periodic checkpoints rather than every step, or (in federated settings) applying Krum over client updates rather than per-example gradients. We leave systematic comparison of these approximations to future work.

\textbf{Runtime overhead.} In our implementation, the signal combination step (softmax blend, correlation check, selection) is negligible ($<2$ms even at $n = 32{,}000$). The overhead comes from the initial model fit used to estimate $q$. Because ANML trains its final model on the selected 70\% subset, this partially offsets the $q$-computation cost. In practice, ANML adds approximately 30--50\% wall-clock overhead at our largest scale (32{,}000 samples), dominated by the quality estimation phase rather than the signal combination.

\textbf{Strategic attacks.} We tested a joint attack combining gradient alignment with credential faking (Section~\ref{sec:gradient-alignment}). Two-Stage Adaptive ANML detects this attack in every trial ($\rho_c = -0.022$, well below the 0.3 threshold) and falls back to Krum, maintaining baseline performance. However, our gradient alignment model uses nearest-class label flipping with centroid-directed feature perturbation; more sophisticated strategies (e.g., backdoor triggers, or attackers with partial knowledge of the honest-contributor distribution) could potentially produce different correlation signatures and warrant further investigation.

\subsection{Fairness Considerations}

ANML relies on verification ($v$) and reputation ($r$) signals, which raises fairness considerations. New contributors start with lower default reputation ($r \approx 0.3$--0.5 in our simulations), creating a potential cold-start disadvantage. We mitigate this through several mechanisms: using moderate defaults rather than near-zero values, allowing reputation to increase rapidly when contribution quality is observed to be high, and relying on the gradient-based quality signal ($q$) to provide an objective counterbalance independent of credentials.

\textbf{Concrete policy lever:} For new contributors with $< N$ prior contributions, we recommend capping the maximum penalty from $r$ by setting $r_{\min} = 0.4$ for the first $N=10$ contributions. This ``exploration bonus'' ensures that newcomers are not overly penalized while still allowing the system to downweight persistently poor contributors. The parameter $N$ and $r_{\min}$ can be tuned per domain. Note that this policy was \emph{not} active during our Quality Differentiation experiments (Section~\ref{sec:quality-diff}): those experiments used fixed signal ranges per contributor tier. Results with the exploration bonus active would likely show smaller expert/novice selection differentials but similar overall error reduction.

Practitioners should ensure that verification criteria do not inadvertently disadvantage underrepresented groups, such as contributors from institutions with less established publication records. The framework's flexibility allows domain-appropriate calibration of these signals.

\subsection{Future Directions}

Several directions warrant further investigation. ANML currently weights positive contributions; companion work on Constraint Gain extends attribution to negative results, failed experiments, and counterfactual knowledge. Integrating both frameworks would enable complete bidirectional attribution---valuing what works and what does not.

The context-aware temporal decay in Eq.~\ref{eq:context-decay} has been partially validated: Section~\ref{sec:temporal-drift} confirms that $T$ improves performance under label drift. Cross-domain validation remains open, where knowledge flows between fields with different decay rates. In principle, the weights $(w_{\text{src}}, w_{\text{dest}}, w_{\text{rel}})$ could themselves be learned from observed reuse patterns, moving from hand-tuned to data-driven temporal modeling.

A further open question is whether internal signals ($q$) and external signals ($v$, $r$) could be co-trained: gradient-based quality estimates could inform which external signals to trust, while external signals could teach the model which gradient patterns correspond to genuinely high-quality data. This co-training dynamic would require longitudinal experiments beyond our current scope.

Additional future work includes proving formal convergence bounds for Adaptive ANML, extending the framework to federated learning with gradient aggregation, implementing online reputation updates during training, and exploring cross-domain reputation transfer using learned embeddings.

\subsection{Reproducibility}

We will open-source our implementation and evaluation framework to facilitate further research. The simulation code, including trust signal generation and all combination methods, will be available upon publication.

\section{Conclusion}

We introduced ANML, a framework that improves machine learning by integrating data quality signals directly into training through multi-factor sample weighting: gradient-based consistency ($q$), verification status ($v$), contributor reputation ($r$), and temporal relevance ($T$). By leveraging what we know about who contributed data and how it was validated, ANML achieves performance improvements that content analysis alone cannot.

We analyzed three signal combination methods and found that both adaptive gating and softmax blend outperform naive multiplicative combination. The choice between them depends on priorities: Adaptive guarantees baseline performance is maintained, while Softmax offers smooth interpolation across quality levels.

Our experiments across 5 UCI datasets (178 to 32,561 samples) show that ANML achieves +72\% improvement on image data and +33--75\% on tabular data by weighting samples according to contributor quality. High-quality data proves more valuable: ANML with 20\% of the training data beats uniform weighting of 100\% by 47\%, and quality differentiation among legitimate contributors yields +44\% improvement. ANML provides inherent robustness to poisoning attacks, including rescue when gradient-only methods fail, with benefits remaining consistent from 178 to 32,000 samples. Quality signals require greater than 50\% correlation to help, achievable with scientific credentials but not with easily-faked signals. Two-Stage Adaptive ANML provides optimal baseline performance across all scenarios: matching Uniform at 0\% attacks (via homogeneity detection) and Krum under adversarial gaming (via correlation checking), including under strategic joint attacks combining credential faking with gradient alignment ($\rho_c = -0.022$, triggering Krum fallback in every trial). The temporal factor $T$ provides validated benefit under label drift, reducing error by 1.7 percentage points when older data carries stale labels. Federated experiments reveal that contributor-level attribution provides 1.3--5.3$\times$ greater improvement than sample-level methods, with a strong negative correlation between detectability and contributor advantage (Pearson $r=-0.93$, $p < 0.001$). This explains why ANML's external signals matter most in realistic settings where quality issues are subtle.

Quality-aware data infrastructure enables ML improvements that gradient analysis alone cannot achieve. While our experiments are limited to UCI-scale datasets and simulated quality signals, the results suggest that ANML provides a building block for systems where frontier AI models train on expert knowledge with full provenance---sustaining the flow of high-quality data by making its value legible within the training pipeline itself.

\section*{Acknowledgments}

We thank Ravin Kumar, Peter Norvig, Anthony Rose, Natalia Mueller-Pena, and Liam Roth for careful reading and comments on draft versions of this paper.

\bibliographystyle{plain}

\begin{thebibliography}{10}

\bibitem{blanchard2017machine}
Blanchard, P., Guerraoui, R., Stainer, J., et al.
\newblock Machine learning with adversaries: Byzantine tolerant gradient descent.
\newblock In \emph{Advances in Neural Information Processing Systems (NeurIPS)}, 2017.

\bibitem{yin2018byzantine}
Yin, D., Chen, Y., Kannan, R., Bartlett, P.
\newblock Byzantine-robust distributed learning: Towards optimal statistical rates.
\newblock In \emph{International Conference on Machine Learning (ICML)}, 2018.

\bibitem{mhamdi2018hidden}
El Mhamdi, E.M., Guerraoui, R., Rouault, S.
\newblock The hidden vulnerability of distributed learning in Byzantium.
\newblock In \emph{International Conference on Machine Learning (ICML)}, 2018.

\bibitem{ghorbani2019data}
Ghorbani, A., Zou, J.
\newblock Data Shapley: Equitable valuation of data for machine learning.
\newblock In \emph{International Conference on Machine Learning (ICML)}, 2019.

\bibitem{koh2017understanding}
Koh, P.W., Liang, P.
\newblock Understanding black-box predictions via influence functions.
\newblock In \emph{International Conference on Machine Learning (ICML)}, 2017.

\bibitem{cao2022fltrust}
Cao, X., Fang, M., Liu, J., Gong, N.Z.
\newblock FLTrust: Byzantine-robust federated learning via trust bootstrapping.
\newblock In \emph{Network and Distributed System Security Symposium (NDSS)}, 2022.

\bibitem{kamvar2003eigentrust}
Kamvar, S.D., Schlosser, M.T., Garcia-Molina, H.
\newblock The EigenTrust algorithm for reputation management in P2P networks.
\newblock In \emph{Proceedings of the 12th International Conference on World Wide Web}, 2003.

\bibitem{kang2019incentive}
Kang, J., Xiong, Z., Niyato, D., Xie, S., Zhang, J.
\newblock Incentive mechanism for reliable federated learning: A joint optimization approach to combining reputation and contract theory.
\newblock \emph{IEEE Internet of Things Journal}, 2019.

\bibitem{agarwal2019marketplace}
Agarwal, A., Dahleh, M., Sarkar, T.
\newblock A marketplace for data: An algorithmic solution.
\newblock In \emph{ACM Conference on Economics and Computation}, 2019.

\bibitem{park2023trak}
Park, S.M., Georgiev, K., Ilyas, A., Lecuyer, G., Kakade, S.M.
\newblock TRAK: Attributing Model Behavior at Scale.
\newblock In \emph{International Conference on Machine Learning (ICML)}, 2023.

\bibitem{kwon2022beta}
Kwon, Y., Zou, J.
\newblock Beta Shapley: A Unified and Noise-reduced Data Valuation Framework for Machine Learning.
\newblock In \emph{International Conference on Artificial Intelligence and Statistics (AISTATS)}, 2022.

\bibitem{wang2023data}
Wang, J.T., Jia, R.
\newblock Data Banzhaf: A Robust Data Valuation Framework for Machine Learning.
\newblock In \emph{International Conference on Artificial Intelligence and Statistics (AISTATS)}, 2023.

\bibitem{killamsetty2021glister}
Killamsetty, K., Sivasubramanian, D., Ramakrishnan, G., Iyer, R.
\newblock GLISTER: Generalization based Data Subset Selection for Efficient and Robust Learning.
\newblock In \emph{AAAI Conference on Artificial Intelligence}, 2021.

\bibitem{ren2018learning}
Ren, M., Zeng, W., Yang, B., Urtasun, R.
\newblock Learning to Reweight Examples for Robust Deep Learning.
\newblock In \emph{International Conference on Machine Learning (ICML)}, 2018.

\bibitem{li2020dividemix}
Li, J., Socher, R., Hoi, S.C.H.
\newblock DivideMix: Learning with Noisy Labels as Semi-supervised Learning.
\newblock In \emph{International Conference on Learning Representations (ICLR)}, 2020.

\bibitem{liu2022gtg}
Liu, Z., Chen, Y., Yu, H., Liu, Y., Cui, L.
\newblock GTG-Shapley: Efficient and Accurate Participant Contribution Evaluation in Federated Learning.
\newblock \emph{ACM Transactions on Intelligent Systems and Technology}, 2022.

\bibitem{yan2023recess}
Yan, S., Chen, Y., Shen, C.
\newblock RECESS: Proactive Detection of Byzantine Servers in Robust Federated Learning.
\newblock In \emph{Advances in Neural Information Processing Systems (NeurIPS)}, 2023.

\bibitem{fortunato2018science}
Fortunato, S., Bergstrom, C.T., B\"orner, K., Evans, J.A., Helbing, D., Milojevi\'c, S., Petersen, A.M., Radicchi, F., Sinatra, R., Uzzi, B., Vespignani, A., Waltman, L., Wang, D., Barab\'asi, A.-L.
\newblock Science of Science.
\newblock \emph{Science}, 359(6379):eaao0185, 2018.

\bibitem{castro2009polynomial}
Castro, J., G\'omez, D., Tejada, J.
\newblock Polynomial Calculation of the Shapley Value Based on Sampling.
\newblock \emph{Computers \& Operations Research}, 36(5):1726--1730, 2009.

\bibitem{bachrach2010approximating}
Bachrach, Y., Markakis, E., Resnick, E., Procaccia, A.D., Rosenschein, J.S., Saberi, A.
\newblock Approximating Power Indices: Theoretical and Empirical Analysis.
\newblock \emph{Autonomous Agents and Multi-Agent Systems}, 20(2):105--122, 2010.

\bibitem{jia2019efficient}
Jia, R., Dao, D., Wang, B., Hubis, F.A., Hynes, N., G\"urel, N.M., Li, B., Zhang, C., Song, D., Spanos, C.J.
\newblock Towards Efficient Data Valuation Based on the Shapley Value.
\newblock In \emph{Proceedings of the 22nd International Conference on Artificial Intelligence and Statistics (AISTATS)}, 2019.

\end{thebibliography}

\appendix

\section{Hyperparameter Sensitivity}

We briefly analyze sensitivity to key hyperparameters.

\subsection{Selection Threshold}

Varying the selection threshold (fraction of samples kept):

\begin{table}[h]
\centering
\begin{tabular}{ccc}
\toprule
Selection \% & Krum Error & ANML Error \\
\midrule
50\% & 9.2\% & 3.8\% \\
70\% (default) & 12.4\% & 4.6\% \\
90\% & 18.7\% & 8.2\% \\
\bottomrule
\end{tabular}
\end{table}

Lower selection (more aggressive filtering) improves both methods but widens ANML's advantage.

\subsection{Adaptive Threshold $\tau_c$}

\begin{table}[h]
\centering
\begin{tabular}{ccc}
\toprule
$\tau_c$ & Worst-Case (low corr) & Best-Case (high corr) \\
\midrule
0.1 & $-15\%$ & +65\% \\
0.3 (default) & +0\% & +63\% \\
0.5 & +0\% & +55\% \\
\bottomrule
\end{tabular}
\end{table}

Higher $\tau_c$ is more conservative (safer fallback) but may miss valid signals.

\section{Statistical Tests}

All reported $p$-values use paired $t$-tests across 100 trials. We verified normality using Shapiro-Wilk tests ($p > 0.05$ for all distributions).

Effect sizes follow Cohen's conventions: $d < 0.2$ negligible, $0.2 \leq d < 0.5$ small, $0.5 \leq d < 0.8$ medium, $d \geq 0.8$ large.

\section{Reproducibility}

Code and data will be made available upon publication.

\end{document}